\def\year{2020}\relax
\documentclass[letterpaper]{article} 
\usepackage{aaai20_no_copyright}
\usepackage{natbib}
\usepackage{times}  
\usepackage{helvet} 
\usepackage{courier}  
\usepackage[hyphens]{url}  
\usepackage{graphicx} 
\usepackage{graphicx}  
\newcounter{vkNoteCounter}

\newcommand{\R}{\mathbb{R}}
\newcommand{\N}{\mathbb{N}}

\newcommand{\mc}{\mathcal}
\newcommand{\mb}{\mathbb}

\newcommand{\defword}[1]{\textbf{\boldmath{#1}}}

\renewcommand{\Pr}{\mathbf{Pr}}
\newcommand{\E}{\mathbf{E}\,}

\newcommand{\terminal}{{\mc T}}

\newcommand{\truth}{{\tau}}
\newcommand{\dGame}{{G}}
\newcommand{\W}{{\mc W}}
\newcommand{\Q}{{\mc Q}}
\newcommand{\A}{{\mc A}}

\newcommand{\Arg}{\mathcal L_c}
\newcommand{\argProt}{{P}}
\newcommand{\supp}{\textrm{supp}\,}

\usepackage{amsfonts,amssymb,amsmath}
\usepackage{amsthm}

\newtheorem{theorem}{Theorem}
\newtheorem{definition}[theorem]{Definition}

\newtheorem{corollary}[theorem]{Corollary}
\newtheorem{remark}[theorem]{Remark}
\newtheorem{problem}[theorem]{Problem}

\usepackage{thm-restate}    
\usepackage{tikz}
\usepackage{enumitem}
    \setitemize{noitemsep,nolistsep,topsep=0pt,parsep=0pt,partopsep=0pt}

\title{(When) Is Truth-telling Favored in AI Debate?}
\author{Vojt\v{e}ch Kova\v{r}\'{i}k\textsuperscript{\rm 1,}\textsuperscript{\rm 2} and Ryan Carey\textsuperscript{\rm 1}\thanks{We are grateful to Chris van Merwijk, Lukas Finnveden, Michael Cohen, and Michael Dennis for feedback and discussions related to this text. The first author was partially supported by MEYS funded project CZ.02.1.01/0.0/0.0/16\_019/0000765 ``Research Center for Informatics''.}\\
\textsuperscript{\rm 1}Future of Humanity Institute\\
University of Oxford\\
\textsuperscript{\rm 2}Artificial Intelligence Center\\
Czech Technical University\\
\emph{vojta.kovarik@gmail.com} 
}

\begin{document}

\maketitle

\begin{abstract}
For some problems, it is difficult for humans to judge the goodness of AI-proposed solutions. \citet{AISvD} propose that in such cases, we may use a debate between two AI systems to assist the human judge to select a good answer.
We introduce a mathematical framework for modelling this type of debate and propose that the quality of debate designs may be measured by the accuracy of the most persuasive answer.
We describe a simple instance of the debate framework called feature debate and analyze the degree to which such debates track the truth.
We argue that despite being very simple, feature debates capture many aspects of practical debates such as the incentives to confuse the judge or stall to prevent losing.
We analyze two special types of debates, those where arguments constitute independent evidence about the topic, and those where the information bandwidth of the judge is limited.
\end{abstract}

\section{Introduction}\label{sec:intro}
In recent years, AI systems have performed impressively in many complex tasks, such as mastering the game of Go \citep{AlphaZero}.
However, these results have largely been limited to tasks with an unambiguous reward function.
To circumvent this limitation, human approval can be used as a measure of success in vague tasks:
For example:
\begin{itemize}
    \item The goodness of a simulated robot backflip is hard to formalize, but an AI system can be trained to maximize the extent to which a human observer approves of its trajectory
    \item The goodness of a film-recommendation is subjective, but an AI system can be trained to maximize the extent to which a human approves of the recommendation.
\end{itemize}

\noindent Unfortunately, once tasks and solutions get too complicated to be fully understood by human users, it is difficult to use human approval to formalize the reward function.
For example, the AlphaGo algorithm could not be trained by maximizing each move's approval since some of its moves looked strange or incorrect to human experts.

\citet{AISvD} suggest addressing this issue by using \emph{AI debate}.
In their proposal, two AI systems are tasked with producing answers to a vague and complex question and then debating the merits of their answers before a human judge.
After considering the arguments brought forward, the human approves one of the answers, thereby allocating reward to the AI system that generated it.
We can apply AI debate to a wide range of questions:
(1) what is the solution of a system of algebraic equations,
(2) which restaurant should I visit today for dinner, or
(3) which of two immigration policies is more socially beneficial.
Moreover, (4) a Go match can be viewed as a debate, where each move is an argument claiming ``my strategy is the better one'', and the winner of the Go-game is called the winner of the debate.

In debates (1) and (4) it is straightforward to ascertain which debater won, and so the most convincing answer always coincides with the most accurate one.
In other debates, such as (2) and (3), misleading arguments may allow a compelling lie to defeat the correct answer.
This raises two central questions for our work: under what circumstances does AI debate track truth? 
And how can debates be designed in order for accurate answers to prevail over less accurate ones?

While researchers have started exploring these questions empirically, the theoretical investigation of AI debate has, at the time of writing this text, mostly been neglected.
The aim of this paper is to begin filling this gap by providing a theoretical framework for reasoning about AI debate, analyzing its basic properties, and identifying further questions that need to be addressed.
So that the work is easy to interpret, we tend to err toward explaining each phenomenon in the simplest model possible, while sketching the extensions necessary to make each model more realistic.

The paper is structured as follows.
Section~\ref{sec:formal}, introduces our model, the \emph{debate game}, and formalizes the problem of designing debate games that promote true answers.
In Section~\ref{sec:feature}, we describe \emph{feature debates}, an instance of the debate game model where the debaters are only allowed to make statements about ``elementary features'' of the world. Section~\ref{sec:guarantees} investigates which feature debates are truth promoting.
Section~\ref{sec:indep_e_and_infolimited_a} continues by analyzing two important subclasses of general debates:
those with ``independent evidence'' and those where the judge's information bandwidth is limited.
Section~\ref{sec:future_work} flags important limitations of the feature debate model and gives suggestions for future work.
Finally, we review relevant literature (Section~\ref{sec:related_work}) and conclude (Section~\ref{sec:conclusion}).
The full proofs are presented in Appendix~\ref{app:proofs}.

\section{The Debate Game Framework}\label{sec:formal}
\subsection{Debate Games}\label{sec:sub:formal}
A debate game (Definition~\ref{def:debate}) is a zero-sum game\footnote{A two-player zero-sum game is one where the utilities satisfy $u_2 = - u_1$. As a result, it suffices to consider the first player's utility $u_1$, and assume that player $2$ is trying to minimize this number.} played between two AI systems that proceeds as follows: the human asks a question about the world, then two AI debaters generate answers and argue that their answer is the better one. Finally, a judge, typically human, uses this dialogue to decide which answer is stronger and allocates the greater share of reward to the debater who produced that answer.

This section formalizes debate games in three steps. First, a definition is given for the \emph{debate environment}  --- the parameters of a debate game that designers generally cannot change --- then the \emph{design elements} --- those parameters that can be changed --- and finally, the debate game. The motivation behind the \emph{debate environment} and \emph{design elements} will be clearer once the debate game is defined (Definition~\ref{def:debate}).

\paragraph{The debate environment}
\begin{definition}[Debate environment]\label{def:env}
A debate environment is a tuple $\mb E = \left< \W, \pi, \Q, \A, \truth, \mc E \right>$ which consists of:
    \begin{itemize}
        \item An arbitrary set $\W$ of \defword{worlds} and a prior distribution \\ $\pi \in \Delta(\W)$ from which the current world is sampled.
        \item A set $\Q$ of \defword{questions}, where each $q \in \Q$ is a text string.
        \item An arbitrary set $\A$ of \defword{answers}.
        \item A mapping $\truth:\Q\times\A\times\W \to [0,\infty)$ which measures the \defword{deviation} of an answer \defword{from the truth} about the world.
        \item A set $\mc E$ of \defword{experiments} $e : \W \to 2^\W$ the judge can perform to learn that the current world $w$ belongs to $e(w)$.
    \end{itemize}
\end{definition}

One example of a debate is a highly general case where $\W$ is the set of all the ways our environment might be, $q\in\Q$ is the set of questions we might ask, and $\A$ are the textual responses that AI debaters might produce. The mapping $\truth$ represents the deviation from the question's true answer, while experiments constitute a cheaper -- and possibly less reliable -- way of obtaining information. For the question \emph{Which restaurant should I go to?}, $\truth(w,q,a)$ could indicate how dissatisfied one would be at each restaurant $a$, and an experiment could indicate my preference between two restaurants by comparing their menus.

We can also consider much more specific cases. For example, $\W$ may represent the set of all legal board positions in Go, and $q$ asks ``What is the optimal next move?''

For many kinds of questions, we may set $\tau(\cdot)=0$ for all correct answers and $\tau(\cdot)=1$ for incorrect or invalid ones.


\paragraph{The design elements of debate}
There are some rules for a debate that a designer can control, called the \emph{design elements}:
\begin{itemize}
    \item the choice of \defword{question} $q$ and \defword{legal answers} $\mc A(q)\subset \mc A$,
    \item \defword{communication language} $\Arg$ (an arbitrary set),
    \item \defword{argumentation protocol} $\argProt : \Q \times \A^2 \times \Arg^* \to 2^{\Arg}$,
    \item \defword{termination} condition $\terminal \subset \mc Q \times \mc A^2 \times \Arg^*$,
    \item \defword{experiment-selection} policy $E : \terminal \to \mc E$, and
    \item \defword{debating incentives} $u_i : \terminal \times 2^\W \to [-1,1]$.
\end{itemize}

In practice, some of these rules will be hard-coded, for example, a designer may restrict the answers to $\mc A = \{\textnormal{``yes''}, \textnormal{``no''}\}$ or make it physically impossible for the debating agents to communicate in anything other than binary code ($\Arg = \{0,1\}^*$).
On the other hand, the designer may outsource the implementation of some rules to a human judge.
For example, a designer can automatically prohibit repetition of a particular string, but the prohibition of rephrasing of previously made points needs to be delegated to the judge. Similarly, a debate can automatically terminate after $N$ rounds, but a judge is needed to end after debaters no longer say anything relevant.\footnote{We can distinguish between the designer of the debate, the person who selects the question $q$, and the judge who determines its winner, who may in-practice be different, or the same person. While the present text does not analyze the role of the human judge in detail, we believe that for such analysis, it is useful to view the judge not as a player in the debate game, but rather as some $J\in \mc J$ which parametrizes $\mc A(q) = \mc A^J(q)$, $\argProt=\argProt^J$, $\terminal=\terminal^J$, $E=E^J$, and $u_i=u_i^J$ (but \emph{not} $\mb E$ and $\Arg$). We do, however, sometimes anthropomorphize parts of debate by speaking as if they were performed by a judge.} 


Using the \emph{debate environment} and \emph{design elements}, a debate game can be formalized as follows:

\begin{definition}[Debate and debate game]\label{def:debate} 
    A \defword{debate} is a tuple $D = \left< \mb E, q, \dGame \right>$, where $\mb E$ is a debate environment, $q\in \Q$ is a question, and $\dGame=(G_q)_{q\in\Q}$ is a \defword{debate game}.
    Formally, each $\dGame_q$ is a two-player zero-sum extensive form\footnote{EFGs are a standard model of sequential decision-making, described, for example, in \citep{osborne1994course}. Partially-observable stochastic games \citep{POSG} constitute an equally valid \citep{FOG} choice of model.} game that proceeds as follows:
    \begin{enumerate}
        \item The world $w$ is sampled from $\pi$ and shown to \defword{debaters} $1$ and $2$ together with the question $q$.
        \item The debaters \emph{simultaneously} pick answers $a_1, a_2 \in \A(q)$.
        \item The debaters alternate\footnote{That is, debater $1$ makes odd arguments $x_1,x_3,\dots$ while $2$ makes $x_2,x_4$, etc.} making \defword{arguments} $x_1, x_2, \ldots \in \Arg$, where $x_j \in \argProt(q,a_1,a_2,x_1,\dots,x_{j-1})$, stopping once $(q,a_1,a_2,x_1,x_2,\dots) = t \in \terminal$ is a \defword{terminated dialogue}.
        \item A single experiment $e=E(t)$ is selected and its result $e(w)$ is used as a context for the next step.
        \item The debaters receive \defword{utilities} $u_1 (t,e(w)) \in [-1,1]$ and $u_2 (t,e(w)) = -u_1 (t,e(w))$.
        \item The answer of the debater with higher utility becomes the \defword{outcome} $o(w,t)$ of $D$ (with ties broken randomly).
        \item \defword{Debate error} is the resulting deviation $\truth (q,w,o(w,t))$ between the outcome of $D$ and the true answer to $q$.
    \end{enumerate}
\end{definition}



The model makes several simplifications, but can be generalized in the following straightforward ways:

\begin{remark}[Natural extensions of debate games]
Debate games may be generalized with:
\begin{itemize}
    \item \textbf{Non-simultaneous answers.} The roles in the answering phase might be asymmetric, in that one debater might see their opponent's answer before selecting their own.\footnote{As in, e.g., the Devil's advocate AI \citep{DAI}.}
    \item \textbf{Debaters with imperfect information.} Instead of having perfect information about $w$, the debaters might only receive some imperfect observation $I_1(w)$, $I_2(w)$. This limitation is particularly relevant for scenarios involving human preferences, such as the restaurant example.
    \item \textbf{Judge interaction.} A judge that asks questions and makes other interjections may be added as a ``chance'' player $J$ with fixed strategy.
    \item \textbf{Dangerous experiments.}
    In the real world, some experiments might have dangerous side-effects. This may be modelled by considering experiments $e : \W \to 2^\W \cup [-\infty,\infty]^3$ which sometimes just give information, but other times bypass the debate and assign utilities $u_1$, $u_2$ and the debate error directly.
    \item \defword{Generalized outcome-selection policies.} Instead of always adopting the more-favoured, the judge may adopt an answer according to a mapping $o : \terminal \times 2^\W \to \mc A$, or may even be given an option of ignoring suspicious or uninformative debates.
\end{itemize}
\end{remark}

\subsection{Properties of debate games}\label{sec:sub:game_theory}
\subsubsection{Debate phases and relation to game theory.}
For the purpose of modelling the debaters' actions, we distinguish between the \defword{answering phase} (step 2 of Definition~\ref{def:debate}) and the \defword{argumentation phase} (step 3).
Once $q$, $w$, and $(a_1,a_2)$ get fixed at the start of the argumentation phase, the utilities of the debaters only depend on the subsequent arguments $x_j$ raised by the agents.
Since the agents have full information about each argument, the \emph{argumentation phase is a two-player zero-sum sequential game with perfect information}. We denote this game $G_{qwa_1a_2}$.
To analyze the answering phase, we first recall an important property of two-player zero-sum games: In every such $G'$, all Nash-equilibrium strategies $\sigma^*$ result in the same expected utility $\E_{\sigma^*} u_1$ \cite[Thm\,3.4.4]{MAS}. This number is called the \defword{value of the game} and denoted $v^*$.
Assuming optimal argumentation strategies, each debater thus knows that playing the argumentation game $G_{qwa_1a_2}$ results in some utility $v^*_{qwa_1a_2}$. This allows them to abstract away the argumentation phase.
By randomizing the order of argumentation and treating both debaters equally, we can further ensure that $v^*_{qwa_2a_1} = - v^*_{qwa_1a_2}$.
As a result, each \emph{answering phase is a symmetric two-player zero-sum matrix game} with actions $\A(q)$ and payoffs $v^*_{wqa_1a_2}$ (to player 1).

These observations have an important implication: Fully general EFGs might contain complications that make finding their solutions difficult. However, both the answering game and the argumentation game belong to highly-specific and well-understood subclasses of EFGs and are thus amenable to simpler solution techniques.

\subsubsection{Measuring the usefulness of debate.}
We measure the suitability of a debate design by the degree to which optimal play by the debaters results in low debate-error.
By default, we focus on the worst-case where both the world and the debate outcome are selected adversarially from the support of their respective distributions. 
We denote the \emph{support} of a probability measure $p$ by $\supp (p)$. 

\begin{definition}[Truth promotion]\label{def:tp}
In the following, $D=\left< \mb E,q,G\right>$ is a debate, $\epsilon\geq 0$, and $w$ always denotes some element of $\textrm{supp}(\pi)$, $\sigma$ a Nash-equilibrium strategy in $G$, and $t$ a terminal dialogue compatible with $\sigma$ in $w$.
$D$ is said to be:
\begin{itemize}
    \item $\epsilon$-\defword{truth promoting} (in the worst-case) in $w$ if for each $\sigma$, we have $\sup_t \truth(q,o(t,w),w) \leq \epsilon$,
    \item $\epsilon$-\defword{truth promoting} if it is $\epsilon$-\defword{truth promoting} in every $w$,
    \item and $\epsilon$-\defword{truth promoting in expectation} if for each $\sigma$, we have $\E_{w\sim\pi} \E_{t\sim\sigma} \, \truth(q,o(t,w),w) \leq \epsilon$.
\end{itemize}
\end{definition}

\noindent When a debate is $0$-truth promoting, we refer to it simply as ``truth-promoting''.
Finally, we formalize the \emph{idealized} version of the design problem as follows:

\begin{problem}[When is debate truth promoting?]\label{prob:characterize_tp}
For given $\left< \mb E, \cdot , G \right>$, characterize those $q\in \Q$ for which any optimal strategy in $\left<\mb E,q,G\right>$ only gives answers with $\truth(q,a,w)=0$.
\end{problem}


\section{Feature Debate}\label{sec:feature}
In order to explore the properties that debates can have, it is useful to have a toy version of the general framework.
In this section, we discuss how questions can be represented as functions of ``elementary features'' of the world and describe a simple debate game in which the arguments are restricted to revealing these features.
This is inspired by \citep{AISvD}, where each world is an image from the MNIST database, elementary features are pixels, and the question is ``Which digit is this?''.
Rather than faithfully capturing all important aspects of debate, the purpose of the toy model is to provide a \emph{simple} setting for investigating \emph{some} aspects. The limitations of the model are further discussed in Section~\ref{sec:future_work}.

\subsection{Defining Feature Debate}\label{sec:sub:feature:toy_model} 

\subsubsection{Questions about functions.}\label{sec:sub:sub:qf}

Many questions are \emph{naturally} expressed as enquiries $q_f$ about some $f : \W \to \mc X$
\begin{equation}
\text{$q_f$ := ``What is output of $f$?''}
\end{equation}
and come accompanied by the answer space $\A = \mc X$ (or $\A = \Delta(\mc X)$).
Examples include questions of measurement (``How far is the Moon?'', ``How much do I weigh?'') and classification (``Which digit is on this picture?'', ``Will person $A$ beat person $B$ in a poker game?'').
For simplicity, we focus on questions $q_f$ about functions $f : \W \to [0,1] = \mc X$ and the truth-mapping $\truth(q,w,a) := |f(w)-a|$.

These assumptions are not very restrictive --- they include binary questions of the type ``Is $Y$ true?'' ($\mc X=\{0,1\}$), and their generalizations ``How likely is $Y$ to be true?'' and ``To what degree is $Y$ true?''.
Any debate about an ``$n$-dimensional question'' can be reduced into $n$ ``$1$-dimensional'' debates, and any function $f : \W \to \R$ can be re-scaled to have range $[0,1]$.


\subsubsection{Feature Debate and Its Basic Properties.}\label{sec:sub:sub:feature}

In feature debate, we assume that worlds are fully described by their \defword{elementary features} --- i.e.\ we suppose that $\W = \Pi_{i =1 }^\infty [0,1]$ and denote $W_i : w=(w_i)_{i=1}^\infty \in \W \mapsto w_i$.
Moreover, we assume that each round consists of each debater making one argument of the form ``the value of $i$-th feature $W_i$ is $x$''.
We consider a judge who can experimentally verify any elementary feature (but no more than one per debate):

\begin{definition}[Feature debate environment]
A \defword{feature debate environment} $\mb F_\pi$ is a debate environment where: 
\begin{itemize}
\item The prior distribution $\pi$ is a (Borel) probability measure on $\W = [0,1]^\N$.
\item Each question pertains to a function $f$, i.e.\ $\Q = \{ q_f \mid f : \W \to [0,1] \textnormal{ measurable} \}$.
\item The answers are $\mc A = [0,1]$.
\item The deviation-from-truth is the distance $\truth(q_f,a,w) = |f(w) - a|$ between the answer $a$ and the truth $f(w)$.
\item Each experiment reveals one feature, i.e.\ $\mc E = \{ e_i \mid i \in \N \}$, where $e_i(w) := \{ \tilde w \mid \tilde w_i=w_i \}$.
\end{itemize}
\end{definition}

For full generality, we may want to assume the debaters can lie about the features, but for the analysis in this paper, we ignore this case. The reason is that if the opponent can point out a lie, and then the judge can test and penalize it, uttering this lie will be sub-optimal.
We thus only consider truthful claims of the form ``$W_i = w_i$''.
With a slight abuse of notation, this allows us to identify the communication language $\Arg$ with the feature-indexing set $\N$.
Any argument sequence $\vec i := \vec i_k := (i_1,\dots,i_k)$ thus effectively reveals the corresponding features. (We write ``$W_{\vec i} = w_{\vec i}$''.)

We suppose the judge has access to the world-distribution $\pi \in \W$ and can update it correctly on new information, but only has ``patience'' to process $2N$ pieces of evidence.
Finally, the debaters are penalized for any deviation between their answer and the judge's final belief and --- to make the debate zero-sum --- are rewarded for any deviation of their opponent.
Adopting the $q_f$ shorthand introduced earlier, the formal definition is as follows:

\begin{definition}[Feature debate]\label{def:feature_debate}
A \defword{feature debate} $F_\pi(f,N) = \left< \mb F_\pi, q_f, G \right>$ is a debate with the following specific rules:
    \begin{itemize}
        \item A randomly selected player makes the first argument.
        \item $\Arg = \N$ and $\argProt(f, a_1, a_2, i_1, \dots, i_k) := \N \setminus \{i_1,\dots,i_k\}$.
        \item After $2N$ arguments have been made, the judge sets $u_1(t,w_{\vec i}) := |\hat f(w_{\vec i}) - a_2| - |\hat f (w_{\vec i}) - a_1|$, where $\hat f(w_{\vec i})$ is the posterior mean
        \begin{equation*}
            \hat f (w_{\vec i}) := \E_\pi \left[ f \mid W_{\vec i} = w_{\vec i} \, \right].\footnote{That is, the judge considers the arguments to be generated independently, using a ``naive-Bayes''-like assumption that ignores the fact that evidence may be selected strategically.}
        \end{equation*}
    \end{itemize}
\end{definition}

\subsection{Optimal play in feature debate}

The zero-sum assumption implies that any shift in $\hat f(w_{\vec i})$ will be endorsed by one debater and opposed by the other (or both will be indifferent). The following symbols denote the two extreme values that the judge's final belief can take, depending on whether the debater who makes the first argument aims for high values of $\hat f(w_{\vec i})$ and the second debater aims for low ones ($\uparrow\downarrow$) or vice versa ($\downarrow\uparrow)$:
\begin{align*}
    \hat f^{\uparrow\downarrow}_N(w) := & \max\nolimits_{i_1} \, \min\nolimits_{i_2} \dots \max\nolimits_{i_{2N-1}} \, \min\nolimits_{i_{2N}} \, \hat f(w_{\vec i}) ,\\
    \hat f^{\downarrow\uparrow}_N(w) := & \min\nolimits_{i_1} \, \max\nolimits_{i_2} \dots \min\nolimits_{i_{2N-1}} \, \max\nolimits_{i_{2N}} \, \hat f(w_{\vec i}).
\end{align*}
Since $\max_x \min_y \varphi(x,y) \leq \min_y \max_x \varphi(x,y)$ holds for any $\varphi$, the second debater always has an edge: $(\forall w \in \W):$ $\hat f^{\uparrow\downarrow}_N(w) \leq \hat f^{\downarrow\uparrow}_N(w)$.
Lemma~\ref{lem:FD_equilibria}\,$(i)$ shows that if the order of argumentation is randomized as in Definition~\ref{def:feature_debate}, the optimal answers lie precisely in these bounds. This result immediately yields a general error bound $(ii)$ which will serve as an essential tool for further analysis of feature debate.
\begin{restatable}[Optimal play in feature debate]{lemma}{NE}\label{lem:FD_equilibria}
(i) The optimal answering strategies in $F_\pi(f,N)$ are precisely all those that select answers from the interval $[\hat f^{\uparrow\downarrow}_N(w), \hat f^{\downarrow\uparrow}_N(w)]$.

(ii) In particular, $F_\pi(f,N)$ is precisely $\max \{ |\hat f^{\uparrow\downarrow}_N(w) - f(w)|, |\hat f^{\downarrow\uparrow}_N(w)-f(w)|\}$-truth promoting in $w$.
\end{restatable}

\section{When Do Feature Debates Track Truth?}\label{sec:guarantees} 

In this section, we assess whether feature debates track truth under a range of assumptions.


\subsection{Truth-Promotion and Critical Debate-Length}\label{sec:sub:basic_props}
Some general debates might be so ``biased'' that no matter how many arguments an honest debater uses, they will not be able to convince the judge of their truth.
Proposition~\ref{prop:triv_guarantees} ensures that this is not the case in a typical feature debate:

\begin{restatable}[Sufficient debate length]{proposition}{trivial}\label{prop:triv_guarantees}
$F_\pi(f,N)$ is truth-promoting for functions that depend on $\leq N$ features.
\end{restatable}

\begin{proof}
In an $N$-round debate about a question that depends on $\leq N$ features, either of the players can unilaterally decide to reveal all relevant information, ensuring that $\hat f(w_{\vec i_{2N}}) = f(w)$. This implies that $\hat f^{\uparrow\downarrow}_N(w) = \hat f^{\downarrow\uparrow}_N(w) = f(w)$. The result then follows from Lemma~\ref{lem:FD_equilibria}.
\end{proof}

\noindent However, Proposition~\ref{prop:triv_guarantees} is optimal in the sense that if the number of rounds is smaller than the number of critical features, the resulting debate error might be very high.

\begin{restatable}[Necessary debate length]{proposition}{impossibility}\label{prop:imposibility}
When $f$ depends on $N+1$ features, the debate error in $F_\pi(f,N)$ can be $1$ (i.e., maximally bad) in the worst-case world and equal to $\frac{1}{2}$ in expectation (even for continuous $f$).
\end{restatable}

\noindent The ``counterexample questions'' which this result relies on are presented in the following Section~\ref{sec:sub:bad_q}.
The formal proof and the continuous case are given in the appendix. 



\subsection{Three Kinds of Very Difficult Questions}\label{sec:sub:bad_q}

We now construct three classes of questions which cause debate to perform especially poorly\footnote{While we focus on results in worst-case worlds, the illustrated behaviour might become the norm with a biased judge (Sec.\,\ref{sec:bias}).}, in ways that are analogous to failures of realistic debates.


\subsubsection{Unfair questions.}
A question may be difficult to debate when \textbf{arguing for one side requires more complex arguments}.
Indeed, consider a feature debate in a world $w$ uniformly sampled from Boolean-featured worlds $\Pi_{i\in \N} W_i = \{0,1\}^\N$, and suppose the debate asks about the conjunctive function $\varphi := W_1 \land \ldots \land W_K$ for some $K\in \N$.
In worlds with $w_1=\dots=w_K=1$, an honest debater has to reveal $K$ features to prove that $\varphi(w) = 1$. On the other hand, a debater arguing for the false answer $a=0$ merely needs to avoid helping their opponent by revealing the features $W_1,\dots,W_K$.
In particular, this setup shows that a debate might indeed require as many rounds as there are relevant features, thus proving the worst-case part of Proposition~\ref{prop:imposibility}).

\subsubsection{Unstable debates.} Even if a question does not bias the debate against the true answer as above, the \textbf{debate outcome} might still be \textbf{uncertain until the very end}.
One way this could happen is if the judge always feels that more information is required to get the answer right. Alternatively, every new argument might come as a surprise to the judge, and be so persuasive that the judge ends up always taking the side of whichever debater spoke more recently.

To see how this behavior can arise in our model, consider the function $\psi := \textnormal{xor}(W_1,\dots,W_K)$ defined on worlds with Boolean features, and the world $w=(1,1,\dots)$.\footnote{Recall that $\psi$ has value $0$ or $1$, depending on whether the number of features $i\leq K$ with $w_i=1$ is even or odd.}
If the world distribution $\pi$ is uniform over $\{0,1\}^\N$, the judge will reason that no matter what the debaters say, the last unrevealed feature from the set $\{W_1,\dots,W_K\}$ always has an equal chance of flipping the value of $\psi$ and keeping it the same, resulting in $\hat \psi(w_{\vec i})=\frac{1}{2}$.
As a result, the only optimal way of playing $F_\pi(\psi,N)$ is to give the wrong answer $a=\frac{1}{2}$, unless a single debater can, by themselves, reveal all features $W_1,\dots,W_K$. This happens precisely when $K\leq N$.
In particular, the case where $K = N+1$ proves the ``in expectation'' part of Proposition~\ref{prop:imposibility}.

To achieve the ``always surprised and oscillating'' pattern, we consider a prior $\pi$ under which each each feature $w_i$ is sampled independently from $\{0,1\}$, but in a way that is skewed towards $W_i=0$ (e.g., $\Pr[W_i=0] = 1 - \delta$ for some small $\delta >0$).
The result of this bias is that no matter which features get revealed, the judge will always be more likely to believe that ``no more features with value $W_i=1$ are coming'' --- in other words, the judge will be very confident in their belief while, at the same time, shifting this belief from $0$ to $1$ and back each round.

\subsubsection{Distracting evidence.} For some questions, there are misleading arguments that appear plausible and then require extensive counter-argumentation to be proven false.
By making such arguments, a dishonest debater can stall the debate, until the judge ``runs out of patience'' and goes with their possibly-wrong surface impression of the topic.
To illustrate the idea, consider the uniform distribution $\pi$ over $\W=[0,1]^\N$ and a question $q_f$ about some $f : \W \to [0,1]$ that only depends on the first $K$ features.
Suppose, for convenience of notation, that the debaters give answers $a_1=1$, $a_2=0$ and the sampled world is s.t. $f(w)=1$ and $w_{K+1}=w_{K+2}=\dots=1$. 
To adversarially modify $f$, we first define a function $S : [0,1]^2 \to [0,1]$ as $S(x,y)=1$ if either $x \neq 1$ or $x = y = 1$ and as $S(x,y)=0$ otherwise.
By replacing $f$ by $f'(w) := f(w) S(w_m,w_n)$, where $n>m>K$, $S$ introduces an ``unlikely problem'' $x=1$ and an ``equally unlikely fix'' $y=1$, thus allowing the dishonest player $2$ to ``stall'' for one round.
Indeed, the presence of $S(\cdot,\cdot)$ doesn't initially affect the expected value of the function in any way.
However, if player $2$ reveals that $W_m = w_m = 1$, the expectation immediately drops to $0$, forcing player $1$ to ``waste one round'' by revealing that $W_n=w_n=1$.
To make matters worse yet, we could ask about $\hat f'(w) := f(w) \Pi_{i=1}^d S(w_{m_i},w_{n_i})$, or use a more powerful stalling function $S(x,y_1 \land \dots \land y_k)$ that forces the honest player to waste $k$ rounds to ``explain away'' a single argument of the opponent.

\subsection{Detecting Debate Failures}
When a debater is certain that their opponent will not get a chance to react, they can get away with making much bolder claims.\footnote{Conversely, some realistic debates might provide first-mover advantage due to anchoring and framing effects.} The resulting ``unfairness'' is not a direct source for concern because the order of play can easily be randomized or made simultaneous. However, we may wish to measure the last-mover advantage in order to detect whether a debate is tracking truth as intended. 
The proof of Lemma~\ref{lem:FD_equilibria} (in particular, equation \eqref{eq:value_calculation}) yields the following result:

\begin{corollary}[Last-mover advantage]\label{cor:last_mover_adv}
If optimal debaters in the feature debate $F_\pi(f,N)$ give answers $a_1,a_2 \in [\hat f^{\uparrow\downarrow}_N(w), \hat f^{\downarrow\uparrow}_N(w)]$, the debater who argues second will obtain $\lvert a_1 - a_2 \rvert$ expected utility.
\end{corollary}

Recall that, by Lemma~\ref{lem:FD_equilibria}, all answers from the interval $[\hat f^{\uparrow\downarrow}_N(w), \hat f^{\downarrow\uparrow}_N(w)]$ are optimal.
Corollary~\ref{cor:last_mover_adv} thus implies that even if the agents debate optimally, some portion of their utility -- up to $\delta := \hat f^{\downarrow\uparrow}_N(w) - \hat f^{\uparrow\downarrow}_N(w)$ -- depends on the randomized choice of argumentation order.\footnote{For illustration, a (literally) extreme case of last-mover advantage occurs in the ``oscillatory'' debate from Section~\ref{sec:sub:bad_q}, where the interval $[\hat f^{\uparrow\downarrow}_N(w), \hat f^{\downarrow\uparrow}_N(w)]$ spans the whole answer space $[0,1]$.}
Incidentally, Lemma~\ref{lem:FD_equilibria} implies that the smallest possible debate error is $\delta/2$ (which occurs when the true answer is $f(w) = ( \hat f^{\uparrow\downarrow}_N(w) - \hat f^{\downarrow\uparrow}_N(w) ) / 2$).
This relationship justifies a simple, common-sense heuristic: If the utility difference caused by reversing the argumentation order is significant, our debate is probably far from truth-promoting.

\section{Two Important Special Cases of Debate}\label{sec:indep_e_and_infolimited_a}
As a general principle, a narrower class of debates might allow for more detailed (and possibly stronger) guarantees.
We describe two such sub-classes of \emph{general} debates and illustrate their properties on variants of feature debate.

\subsection{Debate with Independent Evidence}\label{sec:indep_evidence}

When evaluating solution proposals in practice, we sometimes end up weighing its ``pros'' and ``cons''.
In a way, we are viewing these arguments as (statistically) independent evidence related to the problem at hand.
This is often a reasonable approximation, e.g., when deciding which car to buy, and sometimes an especially good one, e.g., when interpreting questionnaire results from different but independent respondents.
We now show how to model these scenarios as feature debates with statistically independent features, and demonstrate the particularly favourable properties.

\subsubsection{Feature debate with independent evidence.}
As a mathematical model of such setting, we consider $\mc X=\{0,1\}$, $\W := \Pi_{i=1}^\infty [0,1]$, and denote by $W_i : (w,x) \mapsto w_i \in [0,1] $ and $X : (w,x) \mapsto x \in \{0,1\}$ the coordinate projections in $\W \times \mc X$.
Informally, we view the last coordinate as an unknown feature of the world and the debate we construct will be asking ``What is the value of this unknown feature?''.
To enable inference about $X$, we consider some probability distribution $\mb P$ on $\W \times \mc X$. (For convenience, we assume $\mb P$ is discrete.)
Finally, to be able to treat arguments of the form ``$W_i = w_i$'' as independent evidence related to $X$, we assume that the features $W_i$, $i\in \N$, are mutually independent conditional on $X$.\footnote{In other words, $\mb P(W_j=w_j\mid X=x)$ is equal to $\mb P(W_j=w_j\mid X=x, W_{\vec i_k} =w_{\vec i_k})$ for every $x$, $w_j$, and $w_{\vec i_k}$.}
To describe this setting as a feature-debate, we define $\pi$ as the marginalization of $\mb P$ onto $\W$ and consider the question $q=$ ``How likely is $x=1$ in our world?'', i.e. ``what is the value of $f$, where $f(w) := \E_{\mb P} \left[ X \mid \W=w \right]$''.
We denote the resulting ``independent evidence'' feature debate $F_\pi(q,N)$ as $F^{\textnormal{ie}}(\mb P,N)$.



\subsubsection{Judge's belief and its updates.}
Firstly, recall that any probability can be represented using its odds form, which is equivalent to the corresponding log-odds form:
\begin{align*}
    \mb P(A) \in [0,1] & \longleftrightarrow \mb P(A) / \mb P(\neg A) \in [0,\infty] \\
    & \longleftrightarrow \log ( \mb P(A) / \mb P(\neg A)) \in [-\infty,\infty].
\end{align*}
Moreover, when expressed in the log-odds form, Bayes' rule states that updating one's belief in hypothesis $H$ in light of evidence $A$ is equivalent to shifting the log-odds of the prior belief by $\log \left( \mb P\left(A\mid H \right) / \, \mb P\left(A\mid \neg H \right) \right)$.

At any point in the debate $F^{\textnormal{ie}}(\mb P,N)$, the judge's belief $\hat f(w_{\vec i}) = \E [ f \mid W_{\vec i}=w ]$ is, by the definition of $f$, equal to the conditional probability $\mb \mb P(X=1 \mid W_{\vec i}=w_{\vec i})$.
To see how the belief develops over the course of the debate, denote by $\beta_{\vec i}(w)$ the corresponding log-odds form.
Initially, $\hat f_\emptyset(w)$ is equal to the prior $\mb P(X=1)$, which corresponds to $\beta_{\emptyset}(w) = \log ( \mb P(X=1) / \mb P(X=0) ) =: p_0$.
Denoting
\begin{equation*}
\textrm{ev}_j(w) := \log \left( \frac{\mb P(W_j=w_j\mid X=1)}{\mb P(W_j=w_j\mid X=0)} \right) ,
\end{equation*}
the above form of the Bayes' rule implies that upon hearing an argument ``$W_j = w_j$'', the judge will update their belief according to the formula $\beta_{\vec i,j}(w) = \beta_{\vec i}(w) + \textrm{ev}_j(w)$.
In other words, the arguments in $F^\textnormal{ie}(\mb P,N)$ combine additively:
\begin{equation}\label{eq:Fie_update}
    \beta_{\vec i_n}(w) = p_0 + \textrm{ev}_{i_1}(w) + \dots + \textrm{ev}_{i_n}(w) .
\end{equation}

\subsubsection{Optimal strategies and evidence strength.}
Recall that positive (negative) log-odds correspond to probabilities closer to $1$ (resp. 0).
Equation \eqref{eq:Fie_update} thus suggests that for any $w$, the arguments $\N$ can be split into three ``piles'' from which the debaters select arguments: $\mc N_\uparrow := \{ j \in \N \mid \textrm{ev}_j(w) > 0 \}$ containing arguments supporting the answer ``$X=1$ with probability $100\%$'', the pile $\mc N_\downarrow = \{ j \in \N \mid \textrm{ev}_j(w) < 0 \}$ of arguments in favor of the opposite, and the irrelevant arguments $\mc N_{\textnormal{ir}} = \{ j \in \N \mid \textrm{ev}_j(w)=0\}$.
As long as the debaters give different answers, one of them will use arguments from $\mc N_\downarrow$, while the other will only use $\mc N_\uparrow(w)$ (both potentially falling back to $\mc N_{\textnormal{ir}}$ if their pile runs out).\footnote{Formally, these argumentation incentives follow from the first paragraph in the proof of Lemma~\ref{lem:FD_equilibria}.}
Moreover, a rational debater will always use the strongest arguments from their pile, i.e. those with the highest \defword{evidence strength} $|\textrm{ev}_j(w)|$.
Correspondingly, we denote the ``total evidence strength'' that a players can muster in $n$ rounds as 
\begin{align*}
& \textrm{Ev}^\uparrow_n(w) := \max \left\{ \sum\nolimits_{j \in J} \textrm{ev}_{i_k}(w) \mid J\subset \N, |J|=n \right\} \textnormal{ and} \\
& \textrm{Ev}^\downarrow_n(w) := \max \left\{ \sum\nolimits_{j \in J} (- \textrm{ev}_{i_k}(w)) \mid J\subset \N, |J|=n \right\}.
\end{align*}
(To make the discussion meaningful, we assume the evidence sequence $(\textrm{ev}_j(w))_{j}$ is bounded and the maxima above are well-defined.)
The equation \eqref{eq:Fie_update} implies that -- among optimal debaters -- one always selects arguments corresponding to $\textrm{Ev}^\uparrow_N(w)$ while the other aims for $\textrm{Ev}^\downarrow_N(w)$. Since this holds independently of the argumentation order, we get $\hat f^{\uparrow\downarrow}_N(w) = \hat f^{\downarrow\uparrow}_N(w)$.
Together with Lemma~\ref{lem:FD_equilibria}, this observation yields the following result:

\begin{corollary}[Unique optimal answer]
In the answering phase of any $F^{\textnormal{ie}}(\mb P,N)$ with bounded evidence, the only Nash equilibrium is to select the $\hat f^*_N(w)$ which satisfies
\begin{align*}
    \hat f^{\uparrow\downarrow}_N(w) & = \hat f^{\downarrow\uparrow}_N(w) = \hat f^*_N(w) := \textnormal{the probability}\\
    & \! \! \! \! \! \! \! \! \! \! \! \! \! \textnormal{corresponding to the log-odds } p_0 + \textrm{Ev}^\uparrow_N(w) - \textrm{Ev}^\downarrow_N(w) .
\end{align*}
\end{corollary}

\subsubsection{Debate error.}
To compute the debate error in $F^\textnormal{ie}(\mb P,\cdot)$, denote the strength of the evidence that remains in each debater's ``evidence pile'' after $n$ rounds as
\begin{align}
R^\uparrow_n(w) & := \sum\nolimits_{j\in \mc N_\uparrow} \textrm{ev}_j(w) - \textrm{Ev}^\uparrow_n(w), \\
R^\downarrow_n(w) & := \sum\nolimits_{j\in \mc N_\downarrow} (-\textrm{ev}_j(w)) - \textrm{Ev}^\downarrow_n(w) .
\end{align}
Furthermore, assume that additionally to $(\textrm{ev}_j(w))_j$ being bounded, the total evidence $\lim_{n\to\infty} \textrm{Ev}^a_n(w)$ in favor of $a$ is infinite for at most one $a\in \{\uparrow,\downarrow\}$.\footnote{Note that the intuition ``$\lim_n R^{(\cdot)}_n(w) = 0$'' only fits if either $\mc N_a$ or $\mc N_\textnormal{ie}$ is infinite. If a debater eventually has to reveal evidence against their own case, the numbers $R^{(\cdot)}_n(w)$ will get negative.}
Since the true answer $f(w)=\mb P(X=1\mid \W = w)$ corresponds to $p_0 + \sum_{i\in \N} \textrm{ev}_j(w) = p_0 + \sum_{j\in \mc N_\uparrow} \textrm{ev}_j(w) - \sum_{j\in \mc N_\downarrow} (-\textrm{ev}_j(w))$, \emph{the difference between the (log-odds forms of) the judge's final belief and the optimal answer is $R^\uparrow_N(w) - R^\downarrow_N(w)$.}

\subsubsection{Early stopping and online estimation of the debate error.}
If we further assume that the debaters reveal the strongest pieces of evidence first, we can predict a debate's outcome before all $N$ rounds have passed. If the $n$-th argument of player $p$ has strength $|\textrm{ev}_i(w)| =: s_{p,n}(w)$, we know that further $N - n$ rounds of debate cannot reveal more than $(N-n) s_{p,n}(w)$ evidence in favor of $p$.
This implies that as soon as the currently-losing player is no longer able to shift the judge's belief beyond the midpoint between the initial answers, we can stop the debate without affecting its outcome.
If we further know that the question at hand depends on $K$ features or less, we can also bound the difference between $f(w)$ and $\hat f(w_{\vec i})$.
Indeed, in the worst-case scenario, all remaining arguments were all in favor of the same player $p$ --- even in this case, the (log-odds form of) $f(w)$ can be no further than $\max_{p=1,2} s_{p,N}(w) (K - 2N)$ away from the log-odds form of the final belief $\hat f(w_{\vec i_{2N}})$.

\subsection{Debate with Information-Limited Arguments}\label{sec:sub:costly_features}

Sometimes, a single argument cannot convey all relevant information about a given feature of the world.
For example, we might learn that a person $A$ lives in a city $B$, but not their full address, or -- in the language of feature debate -- learn that $w_i$ lies in the interval $[0.5,1]$, rather than understanding right away that $w_i=0.75$.
In such cases, it becomes crucial to model the judge's information bandwidth as limited.

\subsubsection{Feature Debate Representation.}
In feature debate, we can represent each elementary feature $w_i \in [0,1]$ in its binary form (e.g., $(0.75)_2 = 0.11000\dots$), and correspondingly assume that each argument reveals one bit of some $w_i$.
More specifically, we assume that (a) the debaters make arguments of the form ``the $n$-th bit of $i$-th feature has value $b$'', (b) they have to reveal the $n$-th bit of $w_i$ before its $(n+1)$-th bit, and -- using the same argument as in feature debate -- (c) their claims are always truthful.
Informally, each argument in this ``information-limited'' feature\footnote{The name is justified since $\N^2$ is isomorphic to $\N$ and thus $F_\pi^l(f,N)$ is formally equivalent to some feature debate $F_{\tilde \pi}(\tilde f,N)$.} debate $F_\pi^l(f,N)$ thus corresponds to selecting a dimension $i\in \N$ and doing a ``$50\%$ zoom'' on $w$ along this dimension (Figure~\ref{fig:infolimited}).

\begin{figure}[b]
    \centering
        \tikzset{%
            dot/.style={circle, draw, fill=black, inner sep=0pt, minimum width=4pt},
            top/.style={anchor=south, inner sep=5pt},
        }
    \begin{tikzpicture}[thick, scale=2]
        \draw[fill=yellow,line width=0.25mm] (0,0) rectangle (1,1);
        \node[dot] (A) at (0.6,0.25) {};
        \node[top] at (A) {$w$};
    \end{tikzpicture}
    \begin{tikzpicture}[thick, scale=2]
        \draw[,line width=0.15mm] (0,0) rectangle (1,1);
        \draw[fill=yellow,line width=0.25mm] (0.5,0) rectangle (1,1);
        \node[dot] (A) at (0.6,0.25) {};
        \node[top] at (A) {$w$};
    \end{tikzpicture}
    \begin{tikzpicture}[thick, scale=2]
        \draw[,line width=0.15mm] (0,0) rectangle (1,1);
        \draw[,line width=0.15mm] (0.5,0) rectangle (1,1);
        \draw[fill=yellow,line width=0.25mm] (0.5,0) rectangle (0.74,1);
        \node[dot] (A) at (0.6,0.25) {};
        \node[top] at (A) {$w$};
    \end{tikzpicture}
    \begin{tikzpicture}[thick, scale=2]
        \draw[,line width=0.1mm] (0,0) rectangle (1,1);
        \draw[,line width=0.1mm] (0.5,0) rectangle (1,1);
        \draw[,line width=0.1mm] (0.5,0) rectangle (0.74,1);
        \draw[,line width=0.1mm] (0,0) rectangle (1,0.5);
        \draw[fill=yellow,line width=0.25mm] (0.5,0) rectangle (0.74,0.5);
        \node[dot] (A) at (0.6,0.25) {};
        \node[top] at (A) {$w$};
    \end{tikzpicture}
    \caption{Each argument in an information-limited debate reduces the set of feasible worlds by ``zooming-in'' on the sampled world $w = (0.6,0.25)$ along one dimension of $\W$. Here, the first two arguments provide information about $w_1$ (the $x$-axis) and the third one about $w_2$.}
    \label{fig:infolimited}
\end{figure}
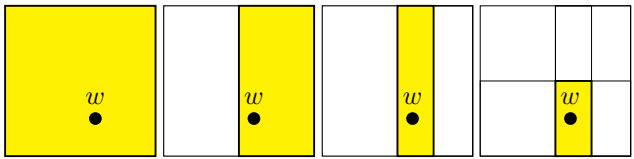

\subsubsection{Performance bounds.}
By offering intermediate steps between features being completely unknown and fully revealed, the debate $F^{\textnormal{ie}}(\cdot)$ allows for more nuanced guarantees than those from Section~\ref{sec:sub:basic_props}.
(Informally stated, the assumptions of the Proposition~\ref{prop:info_limited} can be read as ``the values of $f$ differ by at most $L$ across different worlds, with feature $w_i$ being responsible for up to a $\frac{1}{2^i}$-fraction\footnote{While similar results hold for general ``weight ratios'' between features, we chose weights $\frac{1}{2^i}$ for their notational convenience.} of the variance''.)
\begin{restatable}{proposition}{infoLimited}\label{prop:info_limited}
Suppose that $f : \W \to \R$ is $L$-Lipschitz continuous\footnote{A function is Lipschitz continuous with constant $L \geq 0$ (w.r.t. a metric $\varrho$) if it satisfies $|f(x)-f(y)| \leq L \varrho(x,y)$.} w.r.t. the metric $\rho(w,w')$ = $\sum_{i\in \N} 2^{-i} |w_i-w'_i|$ on $\W$.
Then $F_\pi^l(f,N)$ is $L / 2^{\lfloor \sqrt{N} \rfloor}$-truth promoting.
\end{restatable}

\noindent In contrast to Proposition~\ref{prop:triv_guarantees}, the $L$-Lipschitz assumption thus allows us to get approximately correct debate outcomes long before having full knowledge of all features.
Note that the importance of Proposition~\ref{prop:info_limited} is not in the particular choice of weights, but rather in showing that argument weights can be translated into debate error bounds.

\section{Limitations and Future Work}\label{sec:future_work}

The language introduced so far does not fully capture all aspects of realistic AI debates --- due to space limitations, it is simply not possible to cover all design variants and emergent phenomena in this initial work. 
In this section, we outline some notable avenues for making the debate model more accurate and useful, either by using an alternative instantiation of the general framework from Section~\ref{sec:formal} or by extending the toy model from Section~\ref{sec:feature}.
We start by discussing the modifications which are likely to improve the debate's performance or applicability, and follow-up by those which could introduce new challenges.
For suggested future work on AI debate that is not specific to \emph{modelling}, we refer the reader to \citet{AISvD}.

\subsection{Plausible Improvements to Debate}\label{sec:sub:feature_generalizations}

\subsubsection{Commitments and high-level claims.} 
As \citet{AISvD} suggest, an important reason why debate might work is that the debaters can make abstract or high-level claims that can potentially be falsified in the course of the debate.
For example, debaters might start out disagreeing whether a given image --- of which the judge can only inspect a single pixel --- depicts a dog or a cat. Debater $1$ might then claim that ``here in the middle, there is a brown patch that is the dog's ear'', to which their opponent counters ``the brown patch is a collar on an otherwise white cat''.
Such exchanges would continue until one debater makes a claim that is (i) inconsistent with their answer, (ii) inconsistent with ``commitments'' created by previous arguments, or (iii) specific enough to be verified by the judge.
In this example, (iii) might arise with an exchange ``this pixel is white, which could not happen if it belonged to a brown dog's ear'', ``actually, the pixel is brown'', which allows the judge to determine the winner by inspecting the pixel.

This ability to make high-level claims and create new commitments will often make the debate more time-efficient and incentivize consistency.
Since consistency should typically advantage debaters that describe the true state of the world, commitments and high-level claims seem critical for the success of debate.
We thus need a communication language $\Arg$ that is rich enough to enable more abstract arguments and a set of \defword{effect rules} \citep{prakken2006formal} which specify how new arguments affect the debaters' commitments.
To reason about such debates, we further need a model which relates the different commitments, to arguments, initial answers, and each other.

One way to get such a model is to view $\W$ as the set of assignments for a Bayesian network.
In such setting, each question $q\in\Q$ would ask about the value of some node in $\W$, arguments would correspond to claims about node values, and their connections would be represented through the structure of the network.
Such a model seems highly structured, amenable to theoretical analysis, and, in the authors' opinion, intuitive. It is, however, not necessarily useful for practical implementations of debate, since Bayes networks are computationally expensive and difficult to obtain.

\subsubsection{Detecting misbehaviour.}
One possible failure of debate is the occurrence of stalling, manipulation, collusion, or selective provision of evidence. To remedy these issues, we can introduce specific countermeasures for these strategies. One option for quantifying the contribution of discourse to a human's understanding is to measure the changes in their ability to pass ``exams'' \citep{StuartHumanUnderstanding}. Another countermeasure would be to instantiate a meta-debate on the question of whether a debater is arguing fairly. However, such a meta-debate may, in some cases be even more challenging to judge correctly than the original question.

\subsubsection{Alternative utility functions.}\label{sec:sub:proportional_utils}
We have considered utility functions that are linear in each debater's deviation $\Delta_p := |\hat f(w_{\vec i}) - a_p|$ from the judge's belief.
However, other approaches such as ``divide the total reward in proportion to $\frac{1}{\Delta_p}$'' might give different but interesting results.

\subsubsection{Real-world correspondence.}
To learn which real-world debates are useful on the one hand, and which theoretical issues to address on the other, a better understanding of the correspondence between abstract debate models and real-world debates is needed.
For example, which real-world debates can be modelled as having independent evidence, being Lipschitz, having distracting arguments, and so on?

\subsection{Obstacles in Realistic Debates}

\subsubsection{Sub-optimal judges.}\label{sec:bias}
Some debates might have a canonical idealized way of being judged, which the actual judge deviates from at some steps.
A fruitful avenue for future research is to investigate the extent to which debate fails gracefully as the judge deviates from this ideal.
For example, games are canonically judged as giving a score of $1$ to the winner and $0$ to the loser. We could thus measure how much (and in what ways) the utility function can be modified before the game's winner is changed.
Another approach would be to consider a judge that was biased. An unbiased Bayesian judge would set the prior to the true world-distribution, update the prior on each revealed piece of evidence, and, at the end of the debate, calculate the corresponding expectation over answers. To model a judge who performs the first step imperfectly, we could consider a biased prior $\tilde \pi \in \Delta \W$ (distinct from the true world distribution $\pi$) and calculate the utilities using the corresponding biased belief
    $\tilde \hat f(w_{\vec i}) := \E_{\tilde \pi} \left[ f \mid W_{\vec i} = w_{\vec i} \, \right]$.

\subsubsection{Manipulation.}
So far, we have described failures that come from ``the judge being imperfect in predictable ways''.
However, real-world debates might also give rise to undesirable argumentation strategies inconceivable in the corresponding simplified model.
For example, a debater might learn to exploit a bug in the implementation of the debate or, analogously, find a ``bug'' in the human judge. 
Worse yet, debaters might attempt to manipulate the judge using bribery or coercion. Note that for such tactics to work, the debaters need not be able to carry out their promises and threats --- it is merely required that the judge believes they can.

\subsubsection{Collusion.}
Without the assumption of zero-sum rewards, the debaters gain incentives to collaborate, possibly at the expense of the accuracy of the debate.
Such ``I won't tell on you if you don't tell on me'' incentives might arise, for example, if both agents are given a positive reward if both answers seem good (or negative reward when the debate becomes inconclusive).

\subsubsection{Sub-optimal debaters.}
If debaters argue sub-optimally, we might see new types of fallacious arguments.
We should also expect to see stronger debaters win even in situations that advantage their weaker opponent.
There could also be cases where the losing player complicates the debate game on purpose to increase variance in the outcome.
Both of these phenomena can be observed between humans in games like Go, and we should expect analogous phenomena in general AI debate.
One way of modelling asymmetric capabilities is to let two debaters run the same debating algorithm with a different computational budget (e.g., Monte Carlo tree search with a different number of rollouts).



\section{Related Work}\label{sec:related_work}
\subsubsection{AI safety via debate.}
The kind of debate we sought to model was introduced in \citep{AISvD}, wherein it was proposed as a scheme for safely receiving advice from highly capable AI systems.
In the same work, Irving et al. carried out debate experiments on the MNIST dataset and proposed a promising analogy between AI debates and the complexity class PSPACE.
We believe this analogy can be made compatible with the framework introduced in our Section~\ref{sec:formal}, and deserves further attention.
\citet{ourDebateBlogpost} then demonstrated how to use debate to train an image classifier and described the design elements of debate in more detail.
AI debate is closely related to two other proposals:
(1) ``factored cognition'' \citep{factoredCognition}, in which a human decomposes a difficult task into sub-tasks, each of which they can solve in a manageable time (similarly to how debate eventually zooms in on an easily-verifiable claim), and
(2) ``iterated distillation and amplification'' \citep{christiano2018supervising}, in which a ``baseline human judgement'' is automated and amplified, similarly to how AI debates might be automated.

\subsubsection{Previous works on argumentation.}
Persuasion and argumentation have been extensively studied in areas such as logic, computer science, and law.
The introduction by \citet{prakken2006formal} describes a language particularly suitable for our purposes.
Conversely,  the extensive literature on argumentation frameworks \citep{Dung} seems less relevant. The main reasons are (i) its focus on non-monotonic reasoning (where it is possible to retract claims) and (ii) that it assumes the debate language and argument structure as given, whereas we wish to study the connection between arguments and an underlying world model.
AI systems are also being trained to identify convincing \emph{natural-language} arguments --- for a recent example, see, e.g., \cite{perez2019finding}.

\subsubsection{Zero-sum games.}
As noted in the introduction, we can view two-player zero-sum games as a debate that aims to identify the game's winner (or an optimal strategy).
Such games thus serve as a prime example of a problem for which the state of the art approach is (interpretable as) debate \citep{AlphaZero}.
Admittedly, only a small number of problems are formulated as two-player zero-sum games \emph{by default}.
However, some problems can be reformulated as such games:
While it is currently unclear how widely applicable such ``problem gamification'' is, it has been used for combinatorial problems \citep{xu2019learning} and theorem proving \citep{gameSemantics}.
Together with \citet{AlphaZero}, these examples give some evidence that the AI debate might be competitive (with other problem-solving approaches) for a wider range of tasks.

\section{Conclusion}\label{sec:conclusion}
We have introduced a general framework for modelling AI debates that aim to amplify the capabilities of their judge and formalized the problem of designing debates that promote accurate answers.
We described and investigated ``feature debate'', an instance of the general framework where the debaters can only make statements about ``elementary features'' of the world.
In particular, we showed that if the debaters have enough time to make all relevant arguments, feature debates promote truth.
We gave examples of two sub-classes of debate: those where the arguments provide statistically independent evidence about the answers and those where the importance of different arguments is bounded in a known manner. We have shown that feature debates belonging to these sub-classes are approximately truth-promoting long before having had time to converge fully.
However, we also identified some feature-debate questions that incentivize undesirable behaviour such as stalling, confusing the judge, or exploiting the judge's biases, resulting in debates that are unfair, unstable, and generally insufficiently truth-promoting.
Despite its simplicity, feature debate thus allows for modelling phenomena that are highly relevant to issues we expect to encounter in realistic debates.
Moreover, its simplicity makes feature debate well-suited for the initial exploration of problems with debate and testing of the corresponding solution proposals.
Finally, we outlined multiple ways in which our model could be made more realistic --- among these, allowing debaters to make high-level claims seems like an especially promising avenue.

\section{Acknowledgements}\label{sec:acknowledgements}
This work was supported by the Leverhulme Centre for the Future of Intelligence, Leverhulme Trust, under Grant RC-2015-067.

\bibliography{refs}
\bibliographystyle{aaai}

\appendix

%

\section{Proofs}\label{app:proofs}

We now give the full proofs of the results from the main text.

\begin{proof}[Proof of Lemma~\ref{lem:FD_equilibria}]
Fix $\pi$, $f$, $w$, and $N$, and denote $\Lambda := [\hat f^{\uparrow\downarrow}, \hat f^{\downarrow\uparrow}] := [\hat f^{\uparrow\downarrow}_N(w), \hat f^{\downarrow\uparrow}_N(w)]$.
By definition of $u_1(t) := |\hat f(w_{\vec i}) - a_2| - |\hat f(w_{\vec i}) - a_1|$, the utility $u_1$ is non-decreasing in $\hat f(w_{\vec i})$ when $a_1 \geq a_2$ (and in turn, $u_2$ is non-increasing in $\hat f(w_{\vec i})$).
It follows that \emph{an} optimal argumentation strategy is for player 1 to maximize $\hat f(w_{\vec i})$ and for player 2 to minimize it (and vice versa when $a_1 \leq a_2$).
When calculating utilities, we can, therefore, assume that the final belief is either $\hat f(w_{\vec i}) = \hat f^{\uparrow\downarrow}$ or $\hat f(w_{\vec i}) = \hat f^{\downarrow\uparrow}$, depending on whether the player who argues second gave the higher or lower answer.

Suppose the players gave answers $\{a,b\}$, and the one arguing for $b$ goes second. In this scenario, denote by $v^*(a,b)$ the utility this player receives if both argument optimally.
By applying the above observation (about optimal argumentation strategies) in all possible relative positions of $a$, $b$, and $\hat f^{\uparrow\downarrow} \leq \hat f^{\downarrow\uparrow}$, we deduce that
\begin{equation}\label{eq:value_calculation}
    v^*(a,b) = |a-b| - \textnormal{dist}(b,\Lambda).
\end{equation}

Denote by $v_1^*(a_1,a_2) := \frac{1}{2}v^*(a_2,a_1) - \frac{1}{2}v^*(a_1,a_2)$ the expected utility of player 1 if answers $a_1$ and $a_2$ are given (by players 1 and 2) and players argument optimally (the expectation is w.r.t. the randomized argumentation order).
Using the formula for $v^*(a,b)$, we get
\begin{align*}
    & v_1^*(a_1,a_2) = \frac{1}{2}v^*(a_2,a_1) - \frac{1}{2}v^*(a_1,a_2) \\
    & = \frac{1}{2} \left( |a-b| - \textnormal{dist}(a_1,\Lambda) \right) - \frac{1}{2}( |a-b| - \textnormal{dist}(a_2,\Lambda) ) \\
    & = \frac{1}{2}\textnormal{dist}(a_2,\Lambda) - \frac{1}{2}\textnormal{dist}(a_1,\Lambda).
\end{align*}
It follows that, independently of what strategy the opponent uses, it is always beneficial to give answers from within $\Lambda$ (and that within this interval, all answers are equally good in expectation).\footnote{Incidentally, a more interesting behavior arises if the order of argumentation is fixed; in such a case, player 2 ``flips a coin'' between $\hat f^{\uparrow\downarrow}$ and $\hat f^{\downarrow\uparrow}$ and player 1 picks any of the strategies $\sigma_1$ with $\supp (\sigma_1) \subset \Lambda$ and $\E [ a_1 \mid a_1 \sim \sigma_1 ] = \frac{1}{2}(\hat f^{\uparrow\downarrow} + \hat f^{\downarrow\uparrow})$.}
\end{proof}


\begin{proof}[Proof of Proposition~\ref{prop:imposibility}]
For the worst-case part, suppose that $\pi$ is a uniform distribution over $[0,1]^{\N}$, $w=1$, and let $\varphi := W_1 \land \dots \land W_{N+1}$ be as in Sec.\,\ref{sec:sub:bad_q}.
In $F_\pi(\varphi,N)$, we have $\hat \varphi(w_{\vec i}) = 1$ if $\vec i \supset \{1,\dots,N+1\}$ and $\hat \varphi(w_{\vec i}) = 0$ otherwise.
In particular, we have $\hat \varphi^{\uparrow\downarrow}_N(w) = \hat \varphi^{\downarrow\uparrow}_N(w) = 0$, since the minimizing player can always select $i_k$ from $\N \setminus \{1,\dots,N+1\}$. Since $\varphi(w)=1$, the result follows from Lemma~\ref{lem:FD_equilibria}.

For the ``in expectation'' part, suppose that $\pi$ is a uniform distribution over $\{0,1\}^{N+1}$ (where each $w\in \{0,1\}^{N+1}$ is extended by an infinite sequence of zeros) and $\psi(w) := \textrm{xor}(w_1,\dots,w_{N+1})$.
Recall that the value of $\psi(w)$ will be 1 if the total number of $w_i$-s with $w_i=1$ is odd and $0$ when the number is even.
Until all of the $N+1$ features of any $w$ are revealed, there will be a $50\%$ prior probability of an odd number of them having value $1$ and $50\%$ prior probability of an even number of them having value $1$.
As a result, we have $\hat \psi(w_{\vec i})=\frac{1}{2}$ unless $\vec i \supset \{1,\dots,N+1\}$.
By not revealing any of these features in $F_\pi(\psi,N)$, either of the players can thus achieve $\hat \psi(w_{\vec i})=\frac{1}{2}$.
It follows that $\hat \psi^{\uparrow\downarrow}_N(w) = \hat \psi^{\downarrow\uparrow}_N(w) = \frac{1}{2}$, the only optimal strategy is to give the answer $a=\frac{1}{2}$, and --- since $f$ takes on only values $0$ and $1$ -- the expected debate error is $\frac{1}{2}$.

Our counterexample function $\varphi$ was discontinuous. However, the same result could be achieved by using $f(w) := \Pi_{i=1}^{N+1} w_i$ together with the uniform prior \emph{over $\{0,1\}^\N$} --- this yields a function that is continuous over its (discrete) domain.
Alternatively, we could use the uniform prior over $[0,1]^\N$ together with the continuous approximation $\Pi_{i=1}^{N+1} w_i^K$ (for some large $K\in \N$) of $\varphi$.
Similarly, we could find a continuous approximation of $\psi$.
This observation concludes the proof of the last part of the proposition.
\end{proof}




\begin{proof}[Proof of Proposition~\ref{prop:info_limited}]
Let $w\in \W$.
For $\vec i$, denote by $\W(\vec i)$ the set of worlds $w'$ would reveal the same bits as $w$ under the argument sequence $\vec i$.
To show that the debate is $\epsilon$-truth promoting in $w$, it suffices, by Lemma~\ref{lem:FD_equilibria}, to show that
\begin{equation*}
    f(w) - \epsilon \leq \hat f^{\uparrow\downarrow}_N(w) \textnormal{ and } \hat f^{\downarrow\uparrow}_N(w) \leq f(w) + \epsilon .
\end{equation*}
To get the second inequality, recall that $\hat f^{\downarrow\uparrow}_N(w)$ is defined as the expected value of $f$ under the judge's belief after $N$ rounds of debate, under the assumption that player 1 attempts to drive the expectation as low as possible and player 2 as high as possible.
Suppose that the player who is the first to argue chooses his arguments $i_1$, $i_3$, $i_5$, \dots in a way that minimizes the diameter $\max \{ \rho(w,w') \mid w'\in \W(\vec i) \}$ of $\W(\vec i)$, i.e. by following the sequence $(1, 1, 2, 1, 2, 3, 1, \dots)$.
A simple calculation shows that after $N=1+\dots+n$ rounds, this necessarily gets the diameter of $\W(\vec i)$ to $\frac{n+2}{2^{n+1}}$ or less (depending on the opponent's actions $i_2, i_4, \dots$).
Since $f$ is $L$-Lipschitz, we get that $f(w') \leq f(w) + L \cdot \frac{n+2}{2^{n+1}}$ on $\W(\vec i)$.
Since the optimal strategy of minimizing $\hat f^{\downarrow\uparrow}_N(w)$ can perform no worse than this, we get $\hat f^{\downarrow\uparrow}_N(w) \leq f(w) + L \cdot \frac{n+2}{2^{n+1}}$.

We derive the first inequality analogously. To finish the proof, we need to rewrite the inequality in terms of $N$.
We have $\frac{n+2}{2^{n+1}} \leq \frac{2n}{2^{n+1}} \leq \frac{1}{2^n}$.
Since we have $N \geq 1 + \dots + \lfloor \sqrt{N} \rfloor$ for each $N\in\N$, we get $\hat f^{\downarrow\uparrow}_N(w) \leq f(w) + L \cdot \frac{1}{2^{\lfloor \sqrt{N}\rfloor}}$, which concludes the proof.
\end{proof}

\end{document}